\documentclass[sigconf]{acmart}

\AtBeginDocument{%
  }

\copyrightyear{2026}
\acmYear{2026}
\setcopyright{cc}
\setcctype{by}
\acmDOI{10.1145/3770855.3817938}
\acmConference[KDD '26]{Proceedings of the 32nd ACM SIGKDD Conference on Knowledge Discovery and Data Mining V.2}{August 09--13, 2026}{Jeju Island, Republic of Korea}
\acmBooktitle{Proceedings of the 32nd ACM SIGKDD Conference on Knowledge Discovery and Data Mining V.2 (KDD '26), August 09--13, 2026, Jeju Island, Republic of Korea}
\acmISBN{979-8-4007-2259-2/2026/08}
\usepackage{etoolbox}
\usepackage{algorithm}
\usepackage{algorithmic}
\usepackage{float}
\usepackage{placeins}
\newtheorem{reliftheorem}{Theorem}[section]

\AtBeginDocument{}
\AtBeginDocument{%
  \pdfstringdefDisableCommands{%
    \def\delta{delta}%
    \def\Delta{Delta}%
    \def\lambda{lambda}%
    \def\eta{eta}%
    \def\mu{mu}%
    \def\rho{rho}%
    \def\mathrm#1{#1}%
    \def\mathcal#1{#1}%
  }%
}
\usepackage{booktabs}
\usepackage{enumitem}
\graphicspath{{pic1/}}
\AtBeginDocument{\setkeys{Gin}{draft=false}}
\AtBeginDocument{\let\balance\relax}

\setlist[itemize]{topsep=2pt, itemsep=1pt, parsep=0pt, partopsep=0pt}

\begin{document}

\title[Reliable Lipschitz Fairness in MTL]{Is Fairness Truly Fair? Towards Reliable Lipschitz Fairness in Multi-Task Learning via Fixed-\texorpdfstring{$\delta$}{delta} Alignment}

\author{Junbo Ding}
\authornote{Both authors contributed equally to this research.}
\email{dingjunbo@buct.edu.cn}
\affiliation{%
  \institution{Beijing University of Chemical Technology}
  \city{Beijing}
  \country{China}
}

\author{Xin Zang}
\authornotemark[1]
\email{zangxin@buct.edu.cn}
\affiliation{%
  \institution{Beijing University of Chemical Technology}
  \city{Beijing}
  \country{China}
}

\author{Chenchen Pan} 
\email{panchenchen@buct.edu.cn} 
\affiliation{%
  \institution{Beijing University of Chemical Technology}
  \city{Beijing}
  \country{China}
}

\author{Donghao Song}
\email{songdonghao@buct.edu.cn} 
\affiliation{%
  \institution{Beijing University of Chemical Technology}
  \city{Beijing}
  \country{China}
}

\author{Jiaxin Zhu}
\email{zhujiaxin@buct.edu.cn} 
\affiliation{%
  \institution{Beijing University of Chemical Technology}
  \city{Beijing}
  \country{China}
}

\author{Danhuai Guo}
\authornote{Corresponding author.}
\email{gdh@buct.edu.cn} 
\affiliation{%
  \institution{Beijing University of Chemical Technology}
  \city{Beijing}
  \country{China}
}

\renewcommand{\shortauthors}{Ding et al.}

\begin{abstract}
Lipschitz-style individual fairness formalizes the idea that semantically similar examples should receive similar predictions, but its evaluation in multi-task learning (MTL) can be confounded by method-induced representation scales. This paper identifies threshold confounding: when the auditing tolerance is derived from each model's own representation distances, different algorithms are compared under different semantic thresholds. A threshold-drift analysis further shows how Bias rankings can change and identifies sufficient conditions for ranking preservation.

We propose \textbf{ReLiF}, a reliability-aware framework that separates evaluation-time fixed-$\delta$ auditing from training-time controlled regularization. ReLiF uses a shared reference tolerance for comparable auditing and a violation-rate feedback controller to keep the Lipschitz surrogate active without letting it dominate stochastic training. This work also develops supporting analysis for threshold drift, reference-tolerance selection, and the relationship between the huberized training surrogate and its unsmoothed positive-margin counterpart.

Experiments on clinical time-series benchmarks and NYUv2 (NYU Depth V2) dense prediction show that fixed-$\delta$ auditing exposes utility--fairness trade-offs that method-dependent thresholds can obscure. On NYUv2 with a ResNet50 backbone, ReLiF achieves competitive utility while substantially reducing aligned bias under shared fixed thresholds. On clinical benchmarks, ReLiF yields controlled fairness-regularized trade-offs, while fixed-$\delta$ auditing reveals that task-balancing baselines can sometimes achieve lower bias and that genuine utility--fairness trade-offs persist. These results support fixed-$\delta$ auditing as a semantically consistent protocol for evaluating Lipschitz fairness in MTL.
\end{abstract}

\ccsdesc[500]{Computing methodologies~Multi-task learning}
\ccsdesc[300]{Applied computing~Health informatics}

\keywords{fair multi-task learning; Lipschitz regularization; dynamic control; evaluation alignment; training reliability}


\maketitle
\newcommand\kddavailabilityurl{https://doi.org/10.5281/zenodo.20344205}
\ifdefempty{\kddavailabilityurl}{}{%
\begingroup\small\noindent\raggedright\textbf{Resource Availability:}\\
The source code of this paper has been made publicly available at \url{\kddavailabilityurl}.
\endgroup
}

\section{Introduction}

Multi-task learning (MTL) is widely used when multiple prediction problems share an input representation, ranging from clinical risk prediction to dense prediction in vision. In clinical time-series modeling, one shared encoder may predict mortality, length of stay, decompensation, and phenotype labels from electronic health records~\citep{Harutyunyan2019,johnson2016mimic,Pollard2018eicu}; in vision, one backbone may jointly predict semantic segmentation, depth, and surface normals on NYUv2~\citep{silberman2012indoor,kendall2018uncertainty,liu2019end}. Shared representations can improve efficiency and transfer, but they also create a subtle fairness challenge: prediction discrepancies are often audited through distances induced by the very representation the method itself optimizes.

Lipschitz-style individual fairness formalizes the principle that semantically similar examples should receive similar predictions~\citep{dwork2012fta,pmlr-v119-lohaus20a,jia2024aligning_relational_lipschitz_fairness}. In practice, an audit compares prediction gaps against a tolerance $\delta$ derived from a semantic distance interface. The tolerance is therefore central to what the fairness score measures. If each algorithm induces its own representation scale and therefore its own tolerance, the audit no longer compares all methods under the same semantic standard.

We call this failure mode \textbf{threshold confounding}. A method can appear fairer because it changes the representation scale used to set $\delta$, not because it better controls prediction discrepancies. This is a comparability problem, not merely a hyperparameter-sensitivity issue: once the tolerance is method-induced, the score partly reflects the method's own representation scale rather than its discrepancy control alone. The clinical analysis shows this effect on MTL benchmarks where induced thresholds vary substantially across strong baselines. Appendix~\ref{app:fixed-delta-theory} formalizes this intuition: Bias is Lipschitz in the threshold vector, so a ranking is preserved when the fixed-audit margin exceeds the threshold-drift budget, and any observed reversal implies a fixed-audit gap no larger than that budget. The NYUv2 results in Section~\ref{sec:experiments} further show that shared-threshold auditing remains useful beyond the clinical sequence setting.

A second challenge is that enforcing Lipschitz-style constraints during stochastic non-convex MTL training can be unstable. Fixed penalties may become inactive or overly dominant, producing brittle operating points. ReLiF therefore treats constraint enforcement as a controlled process: the penalty weight is updated from a smoothed violation-rate signal so that the fairness surrogate remains active without overwhelming the task objective.

To address these issues, \textbf{ReLiF} (Reliable Lipschitz Fairness) is introduced as a framework that separates comparable fixed-$\delta$ auditing from training-time controlled regularization. The paper makes three contributions:

\begin{itemize}
    \item \textbf{Threshold-confounding diagnosis and analysis.} We identify threshold confounding in Lipschitz-style MTL fairness auditing and show that threshold drift can change Bias rankings. Our analysis bounds the effect of threshold drift and gives a sufficient condition under which rankings are preserved.
    \item \textbf{Fixed-$\delta$ auditing and controlled training.} We propose ReLiF, which separates evaluation-time fixed-$\delta$ auditing from training-time regularization, uses a shared reference tolerance for comparable reporting, provides a supporting construction for calibrating this tolerance, and uses a violation-rate feedback controller for controlled constraint enforcement.
    \item \textbf{Cross-domain empirical validation.} We evaluate ReLiF on clinical time-series benchmarks and NYUv2 dense prediction with a ResNet50 backbone~\citep{he2016deep}. On NYUv2, ReLiF substantially reduces aligned bias under shared fixed thresholds while maintaining competitive utility; clinical results expose real utility--fairness trade-offs under fixed-$\delta$ auditing.
\end{itemize}

\section{Related Work}
Research on fairness-aware MTL sits at the intersection of multi-objective optimization, Lipschitz-style individual fairness, and evaluation methodology. The central question is whether a fairness audit compares methods under the same semantic tolerance.

\subsection{MTL Optimization and Training Stability}
Modern MTL systems typically rely on shared encoders with task-specific heads, where a primary challenge is mitigating "negative transfer" caused by conflicting gradients or varying loss scales~\citep{Shui2019IJCAI_TaskSimilarity,liu2021conflict}.
Standard solutions rebalance soft objectives. Uncertainty weighting~\citep{kendall2018uncertainty} and GradNorm~\citep{chen2018gradnorm} adjust loss weights, gradient-modification methods such as PCGrad~\citep{yu2020pcgrad} and Recon~\citep{shi2023recon} reduce destructive interference, and multi-objective formulations target Pareto trade-offs~\citep{sener2018pareto_mtl,NEURIPS2019_685bfde0,liu2023famo} or task affinity~\citep{li2024gradtag}. Fairness-oriented MTL optimizers further emphasize explicit resource allocation across tasks~\citep{BanJi2024ICML_FairGrad}. Dense visual MTL benchmarks such as NYUv2 are commonly used to study shared-representation trade-offs among segmentation, depth, and surface normals~\citep{silberman2012indoor,kendall2018uncertainty,liu2019end}. This setting is used not to introduce a new dense-prediction architecture, but to test whether fixed-$\delta$ auditing and ReLiF still apply beyond clinical sequence models.

These methods navigate trade-offs among task utilities, but a fairness constraint must also remain active without overwhelming the task objective~\citep{cotter2019alt_twoplayer,cotter2019jmlr_nondiff,pmlr-v97-cotter19b}. ReLiF therefore studies controlled constraint enforcement alongside task utility.

\subsection{Lipschitz Fairness and the Distance Interface}
Individual fairness requires that similar individuals receive similar predictions~\citep{dwork2012fta,pmlr-v119-mukherjee20a,lahoti2019operationalizing}. 
In practice, this is operationalized via Lipschitz constraints that bound prediction discrepancies relative to a learned semantic distance metric~\citep{jia2024aligning_relational_lipschitz_fairness,hu2023wasserstein}.
Related lines of work study certified or robust individual-fair representations under Lipschitz-style constraints~\citep{NEURIPS2020_55d491cf,Yurochkin2020Training}.
Existing approaches often focus on learning better representations (metric learning)~\citep{pmlr-v119-mukherjee20a,jia2024aligning_relational_lipschitz_fairness} or using task-specific prototypes~\citep{hu2023wasserstein,ding2024feri,jin2024fairmedfm}. Beyond fairness, Lipschitz constraints are also used to improve adversarial robustness and spectral stability~\citep{miyato2018spectral,cisse2017parseval,khromov2024lipschitz_iclr,liu2022siggraph_lipmlp}.

Lipschitz-style auditing is instantiated through a distance--threshold interface: a semantic distance $d(\cdot,\cdot)$ defines which pairs are considered "similar", and a tolerance $\delta$ specifies the similarity radius under which prediction discrepancies should be bounded~\citep{hu2023wasserstein}. 
Most existing work emphasizes improving the distance component, while leaving the tolerance $\delta$ under-specified or method-dependent~\citep{ding2024feri,jin2024fairmedfm}. This lack of control over the semantic tolerance is a primary driver of the evaluation inconsistencies studied in this paper.

\subsection{Evaluation and Threshold Confounding}
The validity of fairness comparisons depends on whether metrics are computed under a shared semantic scale~\citep{wang2024benchmark,gratti2024auc}. A common practice in current pipelines is to adapt the threshold $\delta$ to the model's induced distance distribution for numerical stability~\citep{ding2024feri,jin2024fairmedfm}.
While convenient, this couples evaluation to the method-specific representation scale~\citep{hu2023wasserstein}. As a result, fairness scores can become a "moving target": expanding representation distances can mechanically relax the induced tolerance and reduce hinge-type bias without improving discrepancy control. 
This paper refers to this evaluation artifact as threshold confounding in Lipschitz-style auditing. 
ReLiF restores comparability by using a fixed reference tolerance for evaluation and treats method-induced thresholds only as drift diagnostics. Its controller addresses the separate training problem of keeping the Lipschitz surrogate active in stochastic optimization.

\section{Problem Setup and Metrics}
\label{sec:problem_setup}

This section formalizes the auditing quantities used to distinguish method-induced thresholds from fixed-threshold reporting.
Consider a multi-task prediction setting involving $K$ tasks indexed by $k \in [K]$, where input samples are drawn from a common feature space $\mathcal{X}$ equipped with task-specific supervision. The multi-task model consists of a shared backbone encoder $f_\theta$ and a set of task-specific heads $\{h_k\}_{k=1}^K$. For any input $x \in \mathcal{X}$, head $h_k$ produces a task-specific prediction, from which a scalar audit score $p_k(x) \in [0,1]$ is computed. The same scalar score is used to form both the fixed-audit gaps and the training-time surrogate gaps (Appendix~\ref{app:confidence-computation}), so prediction differences are measured on a unified numeric scale.

\noindent\textbf{Audit-score interface.}
The scalar audit score need not be a calibrated probability for every task. Its role is to put heterogeneous task outputs on a common bounded scale so that cross-task gaps can be computed under a common convention. We require this mapping to be fixed before comparing methods, shared across methods within the same dataset, and evaluated with the same pools and pair-sampling seed. This requirement is especially important for heterogeneous MTL settings, where tasks may have different output types and metric directions. In such cases, utility is still reported with task-native metrics, while fixed-threshold fairness metrics are reported on the shared audit interface.

The standard training objective minimizes the aggregate loss across all tasks:
\begin{equation}
    \min_{\theta} \mathcal{L}_{\text{MTL}}(\theta) = \sum_{k=1}^K L_k(\theta)
    \label{eq:mtl_loss}
\end{equation}
where $L_k(\theta)$ is the task-specific loss. While Eq.~(\ref{eq:mtl_loss}) optimizes for utility, it does not specify cross-task consistency or fairness~\citep{yu2020pcgrad,BanJi2024ICML_FairGrad}.

\subsection{Auditing Interface and Metrics}
Fairness auditing is defined via a distribution $\mathcal{P}_{eval}$ over sampled tuples $(i,j,x,y)$ and a set of semantic thresholds $\{\delta_{ij}\}_{i<j}$, where $1\le i<j\le K$.

\noindent\textbf{Definition 3.1 (Prediction Gap and Violation).}
For a sampled tuple $(i,j,x,y)\sim \mathcal{P}_{eval}$ with $1\le i<j\le K$, the cross-task prediction gap is defined as $\Delta p_{ij}(x, y) = |p_i(x) - p_j(y)|$. A sample satisfies the fairness tolerance if $\Delta p_{ij}(x, y) \leq \delta_{ij}$. The per-pair violation magnitude is:
\begin{equation}
    v_{ij}(x, y) = (\Delta p_{ij}(x, y) - \delta_{ij})_+
    \label{eq:violation_mag}
\end{equation}

Based on this magnitude, two primary metrics are used for Lipschitz fairness:
\begin{itemize}
    \item \textbf{Lipschitz Bias (Bias):} The expected violation magnitude over sampled pairs:
    \begin{equation}
        \text{Bias} = \mathbb{E}_{(i,j,x,y) \sim \mathcal{P}_{eval}} [v_{ij}(x, y)]
    \end{equation}
    \item \textbf{Violation Rate (VR):} The probability of a pair exceeding the threshold:
    \begin{equation}
        \text{VR} = \mathbb{P}_{(i,j,x,y) \sim \mathcal{P}_{eval}} (\Delta p_{ij}(x, y) > \delta_{ij})
    \end{equation}
\end{itemize}

\subsection{The Challenge of Threshold Confounding}
Both Bias and VR are non-increasing in the tolerance levels $\{\delta_{ij}\}$ (see Appendix~\ref{app:fixed-delta-theory}). In standard pipelines, $\{\delta_{ij}\}$ is typically derived from method-induced representation distances. This creates a "moving target": a method that induces larger cosine distances between normalized task prototypes will induce larger $\delta_{ij}$, mechanically lowering the hinge-loss Bias even if the predictor's consistency has not improved. No single $\delta$ is assumed across datasets. A fixed-$\delta$ audit means that within each dataset and evaluation protocol, $\delta$ is selected once from a reference validation distribution and then held fixed for all methods and final reporting. Appendix~\ref{app:fixed-delta-theory} formalizes this observation: the Bias metric is Lipschitz in the threshold vector, and a sufficient condition for preserving a Bias ranking is that the fixed-audit margin is larger than the corresponding threshold-drift budget. Conversely, an observed reversal implies that the fixed-audit gap is no larger than that budget. These results justify treating method-induced thresholds as diagnostics and fixed-$\delta$ scores as the comparable audit.

\subsection{Optimization Goal}
Our objective is to maintain a target \emph{training-time} violation level $r^\star$ for the normalized-margin indicator used by the controller.
Evaluation then reports $\mathrm{Bias}_{\text{fixed }\delta}$ and $\mathrm{VR}_{\text{fixed }\delta}$ under the shared fixed-$\delta$ auditing interface.
\begin{equation}
   r_t \;=\; \frac{1}{|\mathcal{P}_t|}\sum_{(i,j,x,y)\in\mathcal{P}_t}\mathbf{1}\!\left[z_{ij}(x,y)>0\right]\;\approx\; r^\star .
\end{equation}
However, penalty-based enforcement of such constraints in non-convex stochastic optimization can be unstable, often exhibiting oscillations and sensitivity to optimization randomness~\citep{cotter2019alt_twoplayer,cotter2019jmlr_nondiff,gallego_posada2022controlled}.

Consequently, the framework targets two objectives: semantically comparable evaluation via a Fixed-$\delta$ Protocol, and controlled enforcement via feedback control.

\begin{figure*}[!t]
  \centering
  \includegraphics[draft=false,width=\textwidth]{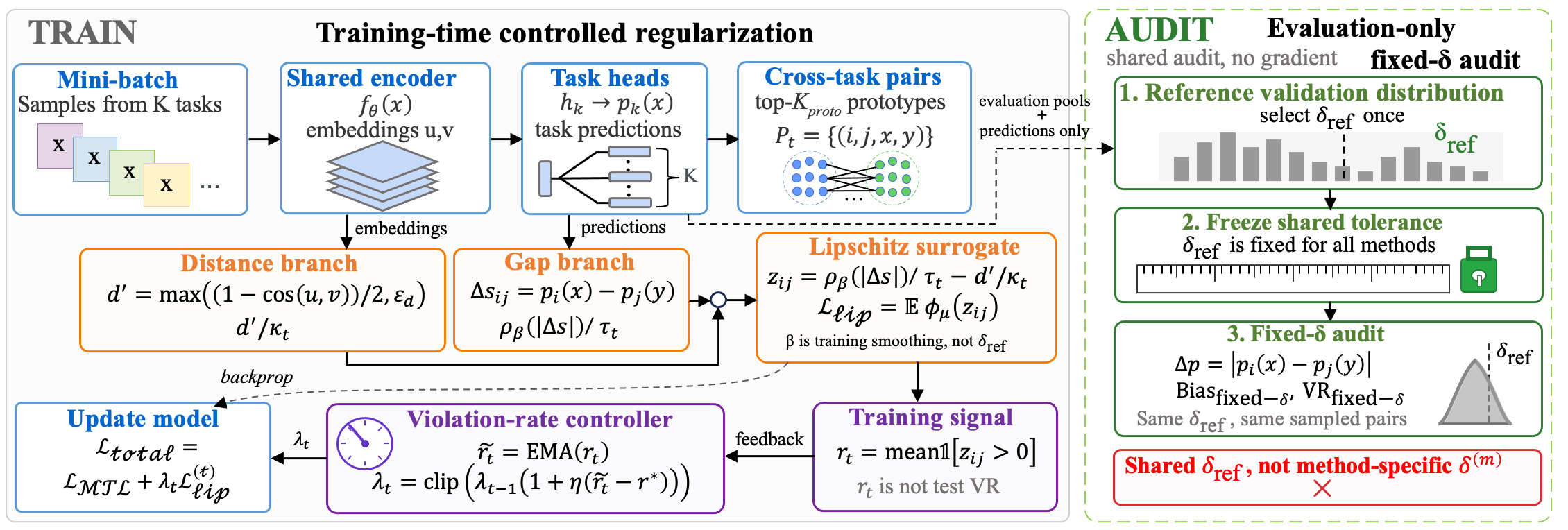}
  \caption{ReLiF separates controlled training from fixed-$\delta$ auditing. During training, mini-batch prototype pairs define normalized margins whose violation rate drives a multiplicative controller for the penalty weight $\lambda$. The model is updated with $\mathcal{L}_{\mathrm{MTL}}+\lambda_t\mathcal{L}_{\mathrm{lip}}^{(t)}$. During evaluation, all methods are audited with the same fixed reference tolerance $\delta_{\mathrm{ref}}$ and the same sampled-pair protocol; no gradient flows through the audit. This separation prevents method-induced representation scales from changing the fairness threshold used for comparison.}
  \Description{A block diagram separating ReLiF training-time controlled regularization from evaluation-only fixed-delta auditing.}
  \label{fig:method}
\end{figure*}

\section{The ReLiF Framework}
\label{sec:ReLiF_framework}

To resolve the dual challenges of threshold confounding and optimization instability, we propose \textbf{ReLiF} (\textbf{Re}liable \textbf{Li}pschitz \textbf{F}airness). ReLiF separates two roles that are often conflated: evaluation uses a fixed reference tolerance, whereas training uses a controlled Lipschitz surrogate whose penalty weight adapts to the current violation signal. Figure~\ref{fig:method} summarizes this separation.

\subsection[Aligned Evaluation via the Fixed-delta Protocol]{Aligned Evaluation via the Fixed-$\delta$ Protocol}
\label{sec:fixed-delta-audit}
Threshold confounding arises when the audit tolerance $\delta$ is derived from the model's own representation distances. To restore semantic consistency, ReLiF evaluates all methods under a Fixed-$\delta$ auditing protocol.

For each dataset and evaluation protocol, $\delta_{\text{ref}}$ is selected once from a reference validation distribution and then held fixed. For clinical benchmarks it is reused for held-out test reporting; for NYUv2 it is reused under the stated dense-prediction evaluation protocol. For any two instances $x$ and $y$ from tasks $i$ and $j$, the prediction gap $|p_i(x)-p_j(y)|$ is audited against the shared tolerance $\delta_{\text{ref}}$, ensuring that all methods are evaluated under the same fairness tolerance within that protocol.

The reference tolerance is a dataset- and protocol-level auditing choice, not a method-specific degree of freedom. In practice, a practitioner specifies a reference representation and a calibration pair distribution, selects the shared tolerance from that reference distribution, and reuses it for all methods. Appendix~\ref{app:fixed-delta-theory} gives an optional split-conformal construction for this step: it calibrates exceedance of reference-pair distances, while the reported fairness metrics still evaluate each method's prediction gaps under the same held-fixed tolerance.

Algorithm~\ref{alg:fixed-delta} summarizes the fixed-$\delta$ auditing procedure used in evaluation. Implementation details of evaluation pool construction, pair sampling, and drift diagnostics are provided in Appendix~\ref{app:dist-sampling}.

\begin{algorithm}[t]
\caption{Fixed-$\delta$ auditing (per run)}
\label{alg:fixed-delta}
\begin{algorithmic}[1]
\REQUIRE Scalar audit scores $\{p_k(\cdot)\}_{k=1}^K$, reference threshold $\delta_{\text{ref}}$, pool size $N_{\mathrm{eval}}$, pairs per task pair $S$, evaluation seed $s_{\mathrm{eval}}$, pair seed $s_{\mathrm{pair}}$
\STATE Construct per-task evaluation pools from valid examples using a fixed evaluation seed
\FOR{each unordered task pair $(i,j)$}
  \STATE Sample $S$ index pairs by independent uniform draws \emph{with replacement} from the two pools
  \STATE Compute gaps $\Delta p_{ij}=|p_i(x)-p_j(y)|$
  \STATE Estimate bias: $\mathrm{Bias}_{ij}\leftarrow \mathrm{mean}\big((\Delta p_{ij}-\delta_{\text{ref}})_+\big)$
  \STATE Estimate VR: $\mathrm{VR}_{ij}\leftarrow \mathrm{mean}\big(\mathbf{1}[\Delta p_{ij}>\delta_{\text{ref}}]\big)$
\ENDFOR
\STATE Aggregate: $\mathrm{Bias}_{\text{fixed }\delta}, \mathrm{VR}_{\text{fixed }\delta} \leftarrow \text{Mean over pairs}$
\STATE Optionally use task prototypes to compute method-induced scales and drift diagnostics for reporting only
\end{algorithmic}
\end{algorithm}

\subsection{Quantile-Normalized Margin via Quantile-Based Scaling}
While a fixed $\delta_{\text{ref}}$ provides a consistent audit, the internal scales of confidence gaps and representation distances can vary across runs due to initialization and loss weighting. To stabilize the training surrogate across steps, ReLiF maintains two running \emph{scale estimates}---a gap scale $\tau_t$ and a distance scale $\kappa_t$---computed from quantile statistics of the current prototype pairs and smoothed by exponential moving average (EMA).

Let $\mathcal{P}_t$ denote the set of mini-batch induced cross-task prototype pairs used by the regularizer at step $t$, where each pair $(x,y)$ is formed by selecting high-confidence prototypes from an unordered task pair $(i,j)$. Distances are computed on the corresponding embeddings $u=f_\theta(x)$ and $v=f_\theta(y)$. The cosine-induced distance is $d(u,v)=(1-\cos(u,v))/2$, and the stabilized training distance is $d'(u,v)=\max\{d(u,v),\epsilon_d\}$ with $\epsilon_d=10^{-3}$.
The 95th percentiles are estimated
\begin{equation}
\begin{aligned}
\widehat{\tau}_t &= \mathrm{Quantile}_{0.95}\Big(\{|\Delta s_{ij}(x,y)| : (i,j,x,y)\in\mathcal{P}_t\}\Big),\\
\widehat{\kappa}_t &= \mathrm{Quantile}_{0.95}\Big(\{d'(f_\theta(x),f_\theta(y)) : (i,j,x,y)\in\mathcal{P}_t\}\Big).
\end{aligned}
\end{equation}
and update running estimates using EMA:
\begin{equation}
    \tau_t \leftarrow (1-\alpha_s)\tau_{t-1}+\alpha_s \widehat{\tau}_t,\qquad
    \kappa_t \leftarrow (1-\alpha_s)\kappa_{t-1}+\alpha_s \widehat{\kappa}_t,
\end{equation}
with a small lower bound to avoid degenerate normalization. This normalization makes the margin comparable across steps and reduces sensitivity to scale changes in either confidence gaps or prototype distances. Appendix~\ref{app:fixed-delta-theory} discusses the corresponding scale-cancellation property for the unsmoothed normalized quantities.

\subsection{Huberized Fairness Surrogate}
The hard indicator in the violation rate is non-differentiable. We reserve $\Delta p_{ij}(x,y)=|p_i(x)-p_j(y)|$ for auditing, and use the signed difference $\Delta s_{ij}(x,y)=p_i(x)-p_j(y)$ inside the training surrogate. ReLiF uses a smooth surrogate based on two Huber-style components: a Huber transform of signed confidence gaps and a huberized hinge on a normalized margin.
Specifically, we robustify gap magnitudes using a Huber function $\rho_{\beta}(\cdot)$ applied to the signed gap $\Delta s_{ij}(x,y)=p_i(x)-p_j(y)$ (equivalently to $|\Delta s_{ij}(x,y)|$ since $\rho_{\beta}$ is even), where $\beta$ controls the transition between the quadratic and linear regimes. The transform is
\begin{equation}
\rho_\beta(u)=
\begin{cases}
u^2/(2\beta), & |u|\le \beta,\\
|u|-\beta/2, & |u|>\beta,
\end{cases}
\label{eq:huber-rho}
\end{equation}
and the huberized hinge
\begin{equation}
\phi_\mu(z)=
\begin{cases}
0, & z\le 0,\\
z^2/(2\mu), & 0<z\le \mu,\\
z-\mu/2, & z>\mu.
\end{cases}
\label{eq:huber-hinge}
\end{equation}
Here $\beta$ and $\mu$ are training-time surrogate hyperparameters and are distinct from the auditing threshold $\delta_{\text{ref}}$ used in the Fixed-$\delta$ Protocol; implementation-level naming details are reported in Appendix~\ref{app:dist-sampling}. With the scale buffers $\tau_t$ and $\kappa_t$, the margin is defined as
\begin{equation}
    z_{ij}(x,y) = \frac{\rho_{\beta}(|\Delta s_{ij}(x,y)|)}{\tau_t} - \frac{d'_{ij}(x,y)}{\kappa_t}.
\end{equation}
Positive margins are then penalized using a huberized hinge $\phi_{\mu}(\cdot)$ with parameter $\mu$~\citep{xu2016huberized_svm}, yielding the Lipschitz regularizer as $\mathcal{L}^{(t)}_{\mathrm{lip}}=\mathbb{E}_{(i,j,x,y)\in\mathcal{P}_t}[\phi_{\mu}(z_{ij}(x,y))]$. Appendix~\ref{app:fixed-delta-theory} states the pointwise approximation bound between the huberized hinge and its unsmoothed positive-margin counterpart.

\subsection{Feedback-Loop Controller}
\label{sec:controller}

Enforcing a target violation rate $r^\star$ in non-convex landscapes is challenging due to the high sensitivity of the objective to the penalty weight $\lambda$. ReLiF adapts $\lambda$ dynamically based on the observed violation signal $r_t$, making constraint enforcement controlled rather than static. The controller consists of two stages: signal smoothing and multiplicative weight updates.

\noindent\textbf{Violation Smoothing:} Since mini-batch violation rates $r_t$ are inherently noisy, we use an EMA to produce a stabilized feedback signal $\tilde{r}_t$:
\begin{equation}
    \tilde{r}_{t} = \alpha_{r} r_{t} + (1 - \alpha_{r}) \tilde{r}_{t-1}, \quad \tilde{r}_{0} = r_{0}
    \label{eq:ema_violation}
\end{equation}
where $\alpha_r \in (0,1)$ is the smoothing rate.

\noindent\textbf{Multiplicative Update Rule:} At training step $t$, after computing the forward-pass violation rate $r_t$, the controller updates the stored weight from $\lambda_{t-1}$ to $\lambda_t$ (after warmup) and applies the new value to the current loss. Its effect on the violation response is observed in subsequent optimization steps. The update uses a multiplicative rule with clipping:
\begin{equation}
    \lambda_{t} = \text{clip}\left( \lambda_{t-1} \cdot (1 + \eta(\tilde{r}_{t} - r^\star)), \lambda_{\min}, \lambda_{\max} \right)
    \label{eq:lambda_update}
\end{equation}
where $\eta$ is the controller step size and $[\lambda_{\min}, \lambda_{\max}]$ defines the operational range. When $\tilde{r}_t > r^\star$, the multiplicative factor increases $\lambda$ and strengthens the regularizer; when $\tilde{r}_t < r^\star$, it decreases $\lambda$ and shifts weight back to the task objective, up to the clipping bounds.
Here $r_t$ is the margin violation rate over $\mathcal{P}_t$, i.e., $r_t=\frac{1}{|\mathcal{P}_t|}\sum_{(i,j,x,y)\in\mathcal{P}_t}\mathbf{1}[z_{ij}(x,y)>0]$, which differs from the fixed-$\delta$ auditing VR.
Appendix~\ref{app:fixed-delta-theory} gives a local mean-square stability statement for the log-penalty controller under explicit monotone-response and bounded-noise assumptions.

\subsection{Training Iteration}
\label{sec:algorithm}

Each training iteration computes $\mathcal{L}_{\text{MTL}}$, forms cross-task prototype pairs, updates the scale buffers $(\tau_t,\kappa_t)$, computes $\mathcal{L}_{\text{lip}}^{(t)}$ and $r_t$, updates $\lambda_t$ after warmup, and optimizes $\mathcal{L}_{\text{MTL}}+\lambda_t\mathcal{L}_{\text{lip}}^{(t)}$. If no valid finite prototype pair is available, the iteration falls back to $\mathcal{L}_{\text{MTL}}$ and leaves the controller state unchanged. Full pseudocode is provided in Appendix~\ref{app:training-pseudocode}.

\begin{table*}[!t]
  \centering
  \small
  \setlength{\tabcolsep}{4pt}
  \renewcommand{\arraystretch}{1.05}
  \caption{NYUv2 results with ResNet50 under shared fixed-$\delta$ auditing. Results are averaged over 7 seeds on the NYUv2 validation-evaluation split; values are mean$\pm$sample standard deviation. Macro and Worst use normalized higher-is-better task scores. Segmentation mean intersection-over-union (mIoU) is higher-is-better, while depth root mean squared error (RMSE) and normal mean error are lower-is-better raw metrics. Bias@$\delta_{p75}$ and Bias@$\delta_{p50}$ use shared thresholds selected once from a reference validation distribution and held fixed for all methods, with pair\_seed=42. In this split, the lowest normalized task score is segmentation for all methods.}
  \label{tab:nyuv2-main}
  \begin{tabular}{lccccccc}
    \toprule
    Method & Worst $\uparrow$ & Macro $\uparrow$ & Seg mIoU $\uparrow$ & Depth RMSE $\downarrow$ & Normal mean $\downarrow$ & Bias@$\delta_{p75}$ $\downarrow$ & Bias@$\delta_{p50}$ $\downarrow$ \\
    \midrule
    ERM & $0.4790\pm0.0026$ & $0.6382\pm0.0016$ & $0.4790\pm0.0026$ & $0.6634\pm0.0082$ & $29.8199\pm0.1312$ & $0.2284\pm0.0185$ & $0.2380\pm0.0200$ \\
    Uncertainty & $\underline{0.4810\pm0.0025}$ & $\underline{0.6395\pm0.0017}$ & $\underline{0.4810\pm0.0025}$ & $0.6609\pm0.0121$ & $\underline{29.6318\pm0.1490}$ & $0.2312\pm0.0231$ & $0.2407\pm0.0222$ \\
    GradNorm & $0.4800\pm0.0029$ & $0.6390\pm0.0017$ & $0.4800\pm0.0029$ & $\underline{0.6599\pm0.0103}$ & $29.7830\pm0.1331$ & $0.2279\pm0.0326$ & $0.2374\pm0.0303$ \\
    PCGrad & $0.4721\pm0.0063$ & $0.6329\pm0.0034$ & $0.4721\pm0.0063$ & $0.6812\pm0.0138$ & $30.3050\pm0.1188$ & $0.2292\pm0.0348$ & $0.2389\pm0.0311$ \\
    FairGrad & $0.4770\pm0.0043$ & $0.6370\pm0.0029$ & $0.4770\pm0.0043$ & $0.6804\pm0.0166$ & $\mathbf{28.9833\pm0.0901}$ & $\underline{0.2141\pm0.0329}$ & $\underline{0.2234\pm0.0310}$ \\
    CATS & $0.4781\pm0.0024$ & $0.6375\pm0.0018$ & $0.4781\pm0.0024$ & $\mathbf{0.6594\pm0.0084}$ & $30.3028\pm0.1064$ & $0.2330\pm0.0335$ & $0.2428\pm0.0327$ \\
    ReLiF & $\mathbf{0.4820\pm0.0030}$ & $\mathbf{0.6398\pm0.0014}$ & $\mathbf{0.4820\pm0.0030}$ & $0.6610\pm0.0092$ & $29.6527\pm0.1401$ & $\mathbf{0.0062\pm0.0090}$ & $\mathbf{0.0085\pm0.0127}$ \\
    \bottomrule
  \end{tabular}
\end{table*}

\section{Experiments}
\label{sec:experiments}

The evaluation covers two domains: clinical time-series prediction and non-medical dense prediction in vision. The experiments answer three questions. \textbf{RQ1:} Does fixed-$\delta$ auditing expose comparable utility--fairness operating points across domains? \textbf{RQ2:} When does ReLiF reduce aligned bias under shared fixed thresholds while maintaining competitive utility? \textbf{RQ3:} What do clinical benchmarks reveal about the trade-offs, sensitivity, and reliability of ReLiF?

\subsection{Experimental Setup}

\noindent\textbf{NYUv2 dense prediction.}
The cross-domain NYUv2 experiment uses a ResNet50 shared backbone~\citep{he2016deep} and three dense prediction tasks: semantic segmentation, depth estimation, and surface normal prediction. We evaluate empirical risk minimization (ERM), Uncertainty, GradNorm, PCGrad, FairGrad, CATS~\citep{liu2023kdd_cats}, and ReLiF under the same training budget. NYUv2 is evaluated on its standard validation split. Results are averaged over 7 seeds and reported with sample standard deviations. Macro and Worst are computed from normalized higher-is-better task scores; raw task metrics are also reported for interpretability. Fixed-threshold auditing uses shared $\delta_{p75}$ and $\delta_{p50}$ values selected from the same reference validation distribution and held fixed for all methods, with pair\_seed=42. ReLiF uses a single controller setting across all NYUv2 seeds; controller settings are reported in Appendix~\ref{app:dist-sampling}.

For NYUv2 auditing, each dense-prediction output is mapped to an image-level scalar audit score shared by all methods. Segmentation uses the valid-region mean max-softmax confidence. Depth and normal use valid-region mean decoder-feature norms, squashed to $[0,1]$ with same-seed ERM validation quantiles. These scalar audit scores instantiate $p_k(x)$ for the fixed-pair auditing protocol; Appendix~\ref{app:confidence-computation} gives the full construction.
This construction is an audit interface rather than a replacement for task metrics: segmentation mean intersection-over-union (mIoU), depth root mean squared error (RMSE), and normal error remain the utility measures, while fixed-$\delta$ Bias and VR compare scalar audit gaps under the same pair protocol. The NYUv2 results should therefore be read as evidence from the standard validation split for cross-domain auditing comparability; richer notions of dense-prediction fairness remain a separate direction.

\noindent\textbf{Clinical benchmarks.}
The clinical benchmarks are the MIMIC-III Clinical Database v1.4~\citep{johnson2016mimic} and the eICU Collaborative Research Database~\citep{Pollard2018eicu}. MIMIC-III follows the Harutyunyan four-task setting~\citep{Harutyunyan2019}: in-hospital mortality, decompensation, phenotype, and length of stay. eICU uses a two-task configuration: mortality and length of stay. Clinical results are reported on the pre-specified held-out test splits, with checkpoint selection performed on validation data only. Results are reported over 7 runs on MIMIC-III and 3 runs on eICU.

\noindent\textbf{Baselines and audit protocol.}
All methods share the same dataset-specific backbone and task heads. The baselines cover standard MTL (ERM, Uncertainty), task-balancing optimizers (GradNorm, PCGrad, CATS), and fairness-aware MTL (FairGrad). Clinical fixed-$\delta$ auditing uses $\delta_{\mathrm{ref}}=0.275$ on MIMIC-III and $\delta_{\mathrm{ref}}=0.06$ on eICU. Method-induced thresholds are used only as diagnostics for threshold drift, not as the comparison metric. Additional implementation details and controller settings are in Appendix~\ref{app:dist-sampling}.

\subsection{NYUv2 Cross-Domain Validation (RQ1--RQ2)}

Table~\ref{tab:nyuv2-main} provides the cross-domain NYUv2 results. ReLiF obtains the highest mean Worst and Macro scores, although the utility differences are small across several baselines. The main separation appears in aligned fairness: under the shared $\delta_{p75}$ audit, ReLiF reduces mean aligned bias to $0.0062$, whereas the strongest non-ReLiF baseline remains above $0.21$. The same pattern holds under $\delta_{p50}$. The aligned-bias separation is therefore not merely a change in rank. Relative to the strongest non-ReLiF baseline, ReLiF reduces Bias@$\delta_{p75}$ from $0.2141$ to $0.0062$ and Bias@$\delta_{p50}$ from $0.2234$ to $0.0085$, while remaining within the same narrow utility band. This evidence suggests that fixed-$\delta$ auditing and controlled regularization remain useful beyond clinical sequence models, in a non-medical dense-prediction setting with a different backbone.

\begin{table}[!t]
  \centering
  \small
  \setlength{\tabcolsep}{4pt}
  \renewcommand{\arraystretch}{1.05}
  \caption{MIMIC-III validation-induced thresholds reveal method-dependent audit scales. Values are mean$\pm$std over 7 validation runs. The table is a threshold-drift diagnostic used to inspect reference-tolerance selection; test fairness reporting uses the held-fixed $\delta_{\mathrm{ref}}=0.275$ audit in Table~\ref{tab:clinical-compact}.}
  \label{tab:mimic-threshold-drift}
  \begin{tabular}{lccc}
    \toprule
    Method & $\bar{\delta}_{\mathrm{val}}$ & $\Delta$ vs ReLiF & Ratio \\
    \midrule
    ReLiF & $\mathbf{0.2752\pm0.0625}$ & \textbf{--} & $\mathbf{1.00\times}$ \\
    GradNorm & $\underline{0.3651\pm0.0604}$ & $\underline{+0.0899}$ & $\underline{1.33\times}$ \\
    FairGrad & $0.3940\pm0.0378$ & $+0.1188$ & $1.43\times$ \\
    Uncertainty & $0.4155\pm0.0504$ & $+0.1403$ & $1.51\times$ \\
    \bottomrule
  \end{tabular}
\end{table}

\subsection{Clinical Trade-Offs and Threshold Confounding (RQ1--RQ3)}

Clinical benchmarks provide complementary trade-offs under the fixed audit. On MIMIC-III, GradNorm achieves the lowest $\mathrm{Bias}_{\mathrm{fixed}\text{-}\delta}$, but its worst-task utility is lower than both ReLiF and FairGrad. FairGrad obtains the strongest MIMIC utility, while ReLiF provides a controlled fairness-regularized alternative with lower aligned bias than FairGrad and higher worst-task utility than GradNorm. On eICU, a simpler two-task setting, ReLiF has the highest Worst score but not the lowest fixed-$\delta$ bias. Several methods nevertheless change ranking after switching from raw to fixed-$\delta$ auditing, indicating that threshold confounding remains an evaluation issue even when ReLiF is not the lowest-bias method.

\begin{table}[!t]
  \centering
  \small
  \setlength{\tabcolsep}{3pt}
  \renewcommand{\arraystretch}{1.05}
  \caption{Clinical fixed-$\delta$ trade-offs. MIMIC-III uses 7 test runs with $\delta_{\mathrm{ref}}=0.275$; eICU uses 3 test runs with $\delta_{\mathrm{ref}}=0.06$. Ranking shifts are computed over all seven methods when switching from raw to fixed-$\delta$ auditing. The table reports representative operating points.}
  \label{tab:clinical-compact}
  \resizebox{\columnwidth}{!}{%
  \begin{tabular}{llcccc}
    \toprule
    Dataset & Method & Worst $\uparrow$ & Macro $\uparrow$ & Bias$_{\mathrm{fixed}\text{-}\delta}$ $\downarrow$ & Shift \\
    \midrule
    MIMIC-III & GradNorm & $0.6330\pm0.0101$ & $\underline{0.7546\pm0.0073}$ & $\mathbf{0.1927\pm0.0099}$ & -- \\
    MIMIC-III & FairGrad & $\mathbf{0.6583\pm0.0038}$ & $\mathbf{0.7672\pm0.0048}$ & $0.2307\pm0.0116$ & -- \\
    MIMIC-III & ReLiF & $\underline{0.6546\pm0.0054}$ & $0.7518\pm0.0058$ & $\underline{0.2237\pm0.0159}$ & -- \\
    \midrule
    eICU & CATS & $\underline{0.6128\pm0.0020}$ & $0.7584\pm0.0008$ & $\mathbf{0.3001\pm0.0066}$ & $7\rightarrow1$ \\
    eICU & GradNorm & $0.6091\pm0.0081$ & $0.7620\pm0.0058$ & $0.3070\pm0.0148$ & $6\rightarrow3$ \\
    eICU & FairGrad & $0.6126\pm0.0023$ & $\mathbf{0.7650\pm0.0015}$ & $\underline{0.3067\pm0.0079}$ & $1\rightarrow2$ \\
    eICU & ReLiF & $\mathbf{0.6139\pm0.0041}$ & $\underline{0.7624\pm0.0029}$ & $0.3155\pm0.0104$ & $5\rightarrow5$ \\
    \bottomrule
  \end{tabular}}
\end{table}

\begin{figure*}[!t]
  \centering
  \includegraphics[draft=false,width=0.8\textwidth]{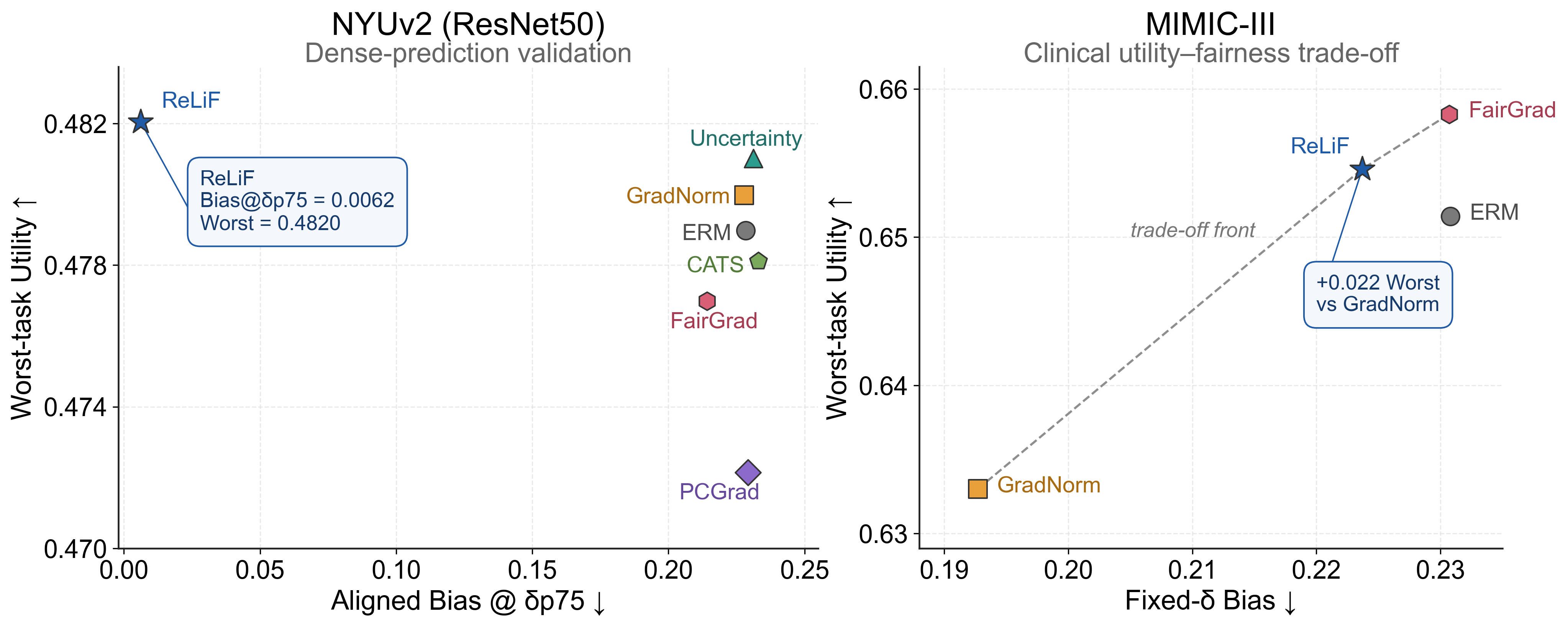}
  \caption{Utility--fairness operating points under shared fixed-$\delta$ auditing. Each marker is one method; the upper-left region is preferred. Left: NYUv2 with ResNet50, using Bias@$\delta_{p75}$. Right: MIMIC-III test split with $\delta_{\mathrm{ref}}=0.275$. NYUv2 shows that ReLiF substantially reduces aligned bias while maintaining competitive utility. MIMIC-III exposes a genuine trade-off: GradNorm obtains the lowest fixed-$\delta$ bias, FairGrad obtains the highest worst-task utility, and ReLiF occupies a controlled point between these two extremes.}
  \Description{Two scatter plots of fixed-delta aligned bias versus worst-task utility for NYUv2 and MIMIC-III.}
  \label{fig:utility-fairness}
\end{figure*}

\noindent\textbf{Threshold-confounding diagnostics.}
Table~\ref{tab:mimic-threshold-drift} shows that MIMIC-III validation-induced thresholds span $0.2752$ to $0.4155$, a $1.51\times$ range across representative methods. This is the stage where the reference tolerance is inspected and fixed before held-out test reporting, so validation-induced thresholds are the appropriate diagnostic for threshold selection. If each method carries its own induced tolerance into evaluation, raw fairness scores are computed under different semantic scales. Because Bias and VR are non-increasing in $\delta$, larger induced thresholds can mechanically reduce raw bias. Method-induced thresholds are therefore treated only as diagnostics, while fixed-$\delta$ bias is used as the comparable auditing metric. The eICU ranking shifts in Table~\ref{tab:clinical-compact} further show that this issue persists beyond the primary benchmark.
The ranking-margin result in Appendix~\ref{app:fixed-delta-theory} explains the mechanism behind such shifts: a Bias ordering can change under method-induced thresholds only when the fixed-audit gap is no larger than the combined threshold-drift budget. Thus, the observed shifts are consistent with the failure mode characterized by the threshold-drift analysis rather than isolated empirical anomalies.

\begin{table}[!t]
  \centering
  \small
  \setlength{\tabcolsep}{3pt}
  \renewcommand{\arraystretch}{1.05}
  \caption{MIMIC-III ReLiF component ablation under the same fixed-$\delta$ test audit. Values are mean$\pm$std over 7 runs. Bold denotes the full ReLiF configuration.}
  \label{tab:mimic-ablation}
  \begin{tabular}{lcc}
    \toprule
    Variant & Worst $\uparrow$ & Bias$_{\mathrm{fixed}\text{-}\delta}$ $\downarrow$ \\
    \midrule
    ReLiF & $\mathbf{0.6546\pm0.0054}$ & $\mathbf{0.2237\pm0.0159}$ \\
    ReLiF w/o Lip & $\underline{0.6508\pm0.0056}$ & $\underline{0.2414\pm0.0187}$ \\
    ReLiF-no-align & $0.6486\pm0.0041$ & $0.2415\pm0.0254$ \\
    \bottomrule
  \end{tabular}
\end{table}

\noindent\textbf{Component evidence.}
Table~\ref{tab:mimic-ablation} isolates the mechanisms in the full ReLiF configuration. Removing the Lipschitz branch increases fixed-$\delta$ bias from $0.2237$ to $0.2414$ while leaving Worst essentially unchanged, showing that the improvement is not solely an artifact of reporting under the fixed audit. Disabling quantile alignment produces a similar bias degradation, supporting the role of scale-normalized margins in keeping the surrogate comparable across training steps. A static-$\lambda$ grid baseline, audited separately and aggregated over 5 penalty values and 7 runs, degrades Worst to $0.6208$ and raises fixed-$\delta$ bias to $0.2550$, consistent with feedback control helping avoid brittle fixed-penalty solutions in this setting.
Together with Table~\ref{tab:clinical-compact}, this ablation separates two messages: ReLiF's components improve the controlled Lipschitz regularizer, but task-balancing baselines can still achieve lower-bias or higher-utility clinical results under the same fixed audit.

\noindent\textbf{Sensitivity and controller checks.}
A wider fixed-threshold sweep on MIMIC-III is reported in Appendix~\ref{app:sensitivity}. As $\delta$ increases, all methods' aligned bias decreases monotonically, as expected from the auditing definition. The qualitative clinical pattern is stable: GradNorm remains the lower-bias method, while ReLiF preserves higher worst-task utility than GradNorm. Controller sensitivity is summarized in Appendix~\ref{app:controller-sensitivity}; the checks are local sensitivity analyses rather than new main results. With moderate $\eta$, all runs train successfully, while overly aggressive $\eta$ can push the controller toward less favorable results. The $r^\star$ sweep similarly moves the utility--fairness trade-off rather than causing immediate failure.

\subsection{Empirical Synthesis}
\label{sec:evidence-chain}

The empirical evidence is organized as a chain rather than as a single leaderboard. Table~\ref{tab:nyuv2-main} is the primary cross-domain result: it changes both the data domain and the backbone relative to the clinical sequence setting, and shows that ReLiF substantially reduces aligned bias under a shared audit while staying in the same narrow utility band as the strongest dense-prediction baselines. Figure~\ref{fig:utility-fairness} then visualizes why the result should be interpreted as a trade-off: the preferred region is lower aligned bias and higher worst-task utility, but different methods can occupy different parts of this region depending on the domain.

The clinical results provide the complementary stress test. Table~\ref{tab:clinical-compact} shows that fixed-$\delta$ auditing does not force a single winner: GradNorm occupies a lower-bias MIMIC point, FairGrad occupies a higher-utility point, and ReLiF provides a controlled fairness-regularized point between them. This is the intended role of the fixed-$\delta$ audit: once the evaluation tolerance is shared, the remaining differences expose real utility--fairness trade-offs that method-induced thresholds can otherwise obscure.

The last two result blocks isolate the two mechanisms behind this interpretation. Table~\ref{tab:mimic-threshold-drift} shows that validation-induced thresholds differ substantially across representative methods before the held-fixed audit is chosen, supporting threshold drift as an evaluation confounder. Table~\ref{tab:mimic-ablation} shows that ReLiF's clinical result is not simply a reporting artifact: removing the Lipschitz branch or disabling quantile alignment worsens fixed-threshold bias under the same audit. Together, the experiments support the paper's main claim: fixed-$\delta$ auditing provides a comparable basis for fairness reporting, and ReLiF is one controlled regularization mechanism that can be useful in shared-representation MTL settings.

\section{Discussion}
\label{sec:discussion}

\subsection{Fixed-threshold auditing}
When $\delta$ is derived from method-induced representation distances, the audit itself becomes method-dependent. Fixed-$\delta$ auditing removes this degree of freedom within a dataset: fairness differences are measured under a shared tolerance rather than under thresholds that each method implicitly adjusts. This does not make $\delta_{\mathrm{ref}}$ a cross-domain constant; it makes the comparison controlled within the stated evaluation protocol.
The formal analysis supports this reporting choice. Bias changes at most linearly with threshold drift. Under the stated audit distribution, any observed reversal of a Bias ranking must therefore have a fixed-audit gap no larger than the corresponding threshold-drift budget. The appendix also provides an optional split-conformal construction for $\delta_{\mathrm{ref}}$: a tolerance can be selected from reference-pair distances and then held fixed for all methods. This makes $\delta_{\mathrm{ref}}$ a reference audit threshold, not a per-method tuning knob. In new settings, practitioners can choose the reference representation and calibration pair distribution according to the intended semantic similarity notion, then reuse the resulting tolerance for all compared methods. Thus, fixed-$\delta$ auditing should be viewed as a comparability protocol: it removes one evaluation degree of freedom before comparing methods.

\subsection{When ReLiF is most useful}
The experiments suggest three regimes. First, when a shared backbone creates strongly coupled task interactions, as in NYUv2 dense prediction, controlled Lipschitz regularization yields the clearest aligned-fairness gain while keeping utility competitive. Second, in clinical MTL with heterogeneous task losses, task-balancing methods such as GradNorm can reduce fixed-$\delta$ bias at the cost of worst-task utility, whereas ReLiF offers a controlled compromise between aggressive discrepancy reduction and worst-task preservation. Third, when task coupling is weaker or the number of tasks is small, as in eICU, the gap between methods under fixed-$\delta$ auditing can shrink and ReLiF need not be the lowest-bias method. The MIMIC-III component ablation supports ReLiF as a controlled regularization design, while the clinical baseline comparisons still show real trade-offs. ReLiF should therefore be viewed as a controlled fairness-regularized framework for comparable bounded-discrepancy auditing, rather than a replacement for pure task balancing in every setting.

\noindent\textbf{Relation to task balancing.}
The clinical results also clarify how ReLiF relates to task-balancing methods. GradNorm, FairGrad, and related methods primarily change the optimization of the task losses, whereas ReLiF controls a separate Lipschitz-style discrepancy penalty under a fixed audit. These mechanisms are therefore not mutually exclusive. A hybrid design could combine a task-balancing rule for $\mathcal{L}_{\mathrm{MTL}}$ with the ReLiF controller for the fairness surrogate. This combined design is not evaluated here because the goal of this paper is to isolate the evaluation confounder and the controlled Lipschitz regularizer. Still, the MIMIC-III results suggest why such a hybrid may be useful: GradNorm reaches lower aligned bias with lower worst-task utility, FairGrad achieves the strongest MIMIC-III utility in this compact comparison, and ReLiF provides a controlled fairness-regularized point between them. Combining these roles is a natural direction for future work.

\subsection{Reporting practice}
For Lipschitz-style MTL fairness, we recommend reporting method-induced thresholds as diagnostics, fixed-$\delta$ Bias/VR as the comparable audit, and controller settings as part of the reproducibility protocol. A complete report should state the source of the reference tolerance, the calibration or validation distribution used to set it, the pair-sampling seed and pool construction, the fixed-threshold Bias/VR values, the method-induced thresholds used only for diagnostics, and the utility metrics for the same checkpoint. These details make it possible to distinguish a method that reduces prediction discrepancies from one that changes the scale of the audit.

This checklist is intentionally protocol-level. It asks authors to disclose the semantic interface and audit construction before comparing methods, rather than treating a fairness score as self-explanatory. In particular, method-induced thresholds should be reported only as drift diagnostics; the comparison itself should be made under the held-fixed tolerance and the same sampled-pair protocol.

\noindent\textbf{Practical workflow.}
In a new application, the protocol can be instantiated in four steps. First, define the scalar audit interface for each task, including how heterogeneous outputs are mapped to comparable scores. Second, choose a reference representation and calibration pair distribution that match the intended semantic similarity notion. Third, select the shared tolerance from this reference distribution and hold it fixed for all methods and all final comparisons. Fourth, report utility, fixed-threshold Bias/VR, method-induced thresholds as diagnostics, and controller settings together. This workflow separates three decisions that are often conflated: the semantic interface used to define comparable pairs, the fixed tolerance used for evaluation, and the training mechanism used to reach a useful solution.

\subsection{Limitations}
The protocol and method have clear boundaries. First, fixed-$\delta$ auditing controls within-protocol comparability, but it does not establish cross-domain semantic equivalence. Choosing a reference representation and calibration pair distribution remains a domain decision. Second, the dense-prediction audit scores used for NYUv2 are scalar protocol proxies over per-pixel outputs; they preserve the audit interface but do not capture every notion of fairness in dense prediction.

Third, the controller analysis in Appendix~\ref{app:fixed-delta-theory} is local in the log-penalty coordinate under explicit monotone-response and bounded-noise assumptions; it is not a global convergence result for non-convex constrained training. Fourth, ReLiF is not the lowest-bias method in every clinical regime. When task balancing aligns well with the fixed-$\delta$ audit, methods such as GradNorm can reduce aligned bias at the cost of worst-task utility.

These limitations define the intended scope of the protocol and method. The paper targets comparable auditing and controlled Lipschitz enforcement, not method dominance across all settings, a single semantic tolerance for every domain, or full semantic fairness certification.

\section{Conclusion}
\label{sec:conclusion}

This paper identified threshold confounding as an evaluation pitfall in Lipschitz-style fairness for multi-task learning: when each method induces its own tolerance, reported fairness scores may reflect different semantic standards. ReLiF addresses this issue by separating evaluation-time fixed-$\delta$ auditing from training-time controlled Lipschitz regularization.

Across clinical time-series benchmarks and NYUv2 dense prediction, fixed-$\delta$ auditing exposes utility--fairness operating points that method-induced thresholds can obscure. The NYUv2 results show that ReLiF can substantially reduce aligned bias under shared thresholds while maintaining competitive utility in a non-clinical vision setting. The clinical results clarify the scope of the claim: ReLiF is not the lowest-bias method in every regime, but the fixed-$\delta$ protocol makes these trade-offs visible and comparable.

The formal analysis supports this protocol in four ways: it bounds the effect of threshold drift, provides a construction for calibrating the reference tolerance, relates the training surrogate to the normalized margin, and characterizes a locally stable controller regime under explicit assumptions. Future work should extend fixed-$\delta$ auditing to richer semantic similarity notions, including subgroup-conditional or time-varying tolerances, and characterize when controller settings transfer across domains.

\begin{acks}
This research was partially funded by the National Natural Science Foundation of China under Grant No. 42371476. We thank the reviewers for their constructive comments.
\end{acks}

\FloatBarrier
\bibliographystyle{ACM-Reference-Format}
\bibliography{refs}

\FloatBarrier
\appendix

\section*{Appendix}

\section{Training and Audit Details}
\label{app:dist-sampling}

\subsection{Training pseudocode}
\label{app:training-pseudocode}
Algorithm~\ref{alg:relif-training-appendix} summarizes the implementation-level training loop. The controller operates on the training-time normalized margin signal; the fixed-$\delta$ audit is a separate evaluation protocol and does not backpropagate through the model.
\subsection{Implementation settings}
\label{app:implementation-settings}
Table~\ref{tab:hyperparams} reports the implementation and reporting
settings used in the experiments.

\subsection{Confidence computation}
\label{app:confidence-computation}
For clinical classification tasks, task outputs are converted to scalar audit scores in $[0,1]$: sigmoid probability for binary mortality, mean sigmoid probability across valid decompensation time steps, maximum softmax probability for length-of-stay bins, and maximum sigmoid probability for multi-label phenotype prediction. For NYUv2, auditing is image-level: the shared representation is obtained by global average pooling the encoder feature map. Segmentation uses the valid-region mean of per-pixel max-softmax probabilities. Depth and normal use the valid-region mean channel-wise $\ell_2$ norm of the decoder feature map, squashed to $[0,1]$ by $\mathrm{clip}((s-q_{10})/(q_{90}-q_{10}+10^{-6}),0,1)$, where $q_{10}$ and $q_{90}$ are same-seed ERM validation quantiles. These audit-score definitions are shared by all methods under the fixed pair-sampling protocol.

\begin{algorithm}[!t]
\caption{ReLiF training loop}
\label{alg:relif-training-appendix}
\begin{algorithmic}[1]
\REQUIRE Mini-batches, MTL loss $\mathcal{L}_{\mathrm{MTL}}$, warmup $T_w$, bounds $[\lambda_{\min},\lambda_{\max}]$, target $r^\star$, step size $\eta$.
\STATE Initialize $\lambda_0$, scale buffers $\tau_0,\kappa_0$, and violation EMA $\widetilde r_0$.
\FOR{training step $t=1,2,\ldots$}
  \STATE Compute task losses and $\mathcal{L}_{\mathrm{MTL}}$.
  \STATE Form cross-task prototype pairs and compute signed confidence gaps $\Delta s_{ij}=p_i(x)-p_j(y)$ and clipped distances $d'$.
  \STATE Update stop-gradient scale buffers using $\mathrm{Quantile}_{0.95}(|\Delta s|)$ and $\mathrm{Quantile}_{0.95}(d')$ with EMA.
  \STATE Compute normalized margins $m=\rho_\beta(|\Delta s|)/\tau_t-d'/\kappa_t$, loss $\mathcal{L}_{\mathrm{lip}}=\mathbb{E}[\phi_\mu(m)]$, and training violation rate $r_t=\mathbb{P}(m>0)$.
  \STATE Update $\widetilde r_t\leftarrow\mathrm{EMA}(r_t)$; if $t>T_w$, set $\lambda_t\leftarrow\mathrm{clip}(\lambda_{t-1}(1+\eta(\widetilde r_t-r^\star)))$.
  \STATE Optimize $\mathcal{L}_{\mathrm{MTL}}+\lambda_t\mathcal{L}_{\mathrm{lip}}$.
\ENDFOR
\end{algorithmic}
\end{algorithm}

\begin{table}[!t]
\centering
\scriptsize
\setlength{\tabcolsep}{3pt}
\renewcommand{\arraystretch}{1.05}
\caption{Core implementation and reporting settings. Controller settings are reported for reproducibility; they are not used as audit thresholds.}
\label{tab:hyperparams}
\begin{tabular}{p{0.30\columnwidth}p{0.64\columnwidth}}
\toprule
Component & Setting \\
\midrule
Clinical model & 2-layer long short-term memory (LSTM), hidden size 256, dropout 0.3; AdamW, lr $10^{-3}$, weight decay $10^{-5}$, batch size 64, gradient clipping 1.0. \\
NYUv2 model & ResNet50 shared backbone with segmentation, depth, and normal heads; all methods use the same backbone and training budget. \\
Distances & Cosine-induced distance $(1-\cos(a,b))/2$ on $\ell_2$-normalized prototype embeddings; training distances are clipped at $10^{-3}$. \\
Scale buffers & 95th-percentile normalization with EMA weight 0.1 for $\tau_t$ and $\kappa_t$; buffers are stop-gradient and clamped at $10^{-4}$. \\
Surrogate & Huber parameter $\beta=0.1$ and hinge-smoothing parameter $\mu=0.05$; implementation-level \texttt{delta} denotes $\beta$, not the audit threshold $\delta_{\mathrm{ref}}$. \\
Controller & $\lambda_0=0.1$, $\lambda_{\min}=0.01$, $\lambda_{\max}=1.0$, warmup 100 steps, violation EMA weight 0.1. \\
ReLiF settings & NYUv2 uses $r^\star=0.24,\eta=10^{-3}$ for all seven seeds. eICU uses $r^\star=0.16,\eta=10^{-3}$. MIMIC-III selects $(r^\star,\eta)$ on validation per data-seed group before held-out test reporting; selected values lie within the local controller range examined in Appendix~\ref{app:controller-sensitivity}. \\
Audit sampling & Unless stated otherwise, fixed-$\delta$ audit uses $N_{\mathrm{eval}}=256$, $S=4096$, and pair\_seed=42. \\
\bottomrule
\end{tabular}
\end{table}

\section{Additional Analysis}
\label{app:sensitivity}

\subsection[Formalized Analysis of Fixed-delta Auditing and Control]{Formalized Analysis of Fixed-$\delta$ Auditing and Control}
\label{app:fixed-delta-theory}
This appendix formalizes the analysis sketched in the main text. The results below do not assert global convergence of non-convex constrained training. They state the audit-comparability properties of fixed thresholds, a construction for calibrating the reference tolerance, the connection between the huberized surrogate and its unsmoothed positive-margin counterpart, and a local controller stability statement under explicit response assumptions.

Let an audit tuple be $a=(i,j,x,y)$ with $i<j$ and $a\sim P_{\mathrm{eval}}$. For method $m$, write
$\Delta_m(a)=|p_i^m(x)-p_j^m(y)|\in[0,1]$. For threshold vector $\boldsymbol{\delta}=(\delta_{ij})_{i<j}$, define
\begin{equation}
\begin{aligned}
B_m(\boldsymbol{\delta})
&=\mathbb E_{a\sim P_{\mathrm{eval}}}
\!\left[(\Delta_m(a)-\delta_{ij(a)})_+\right],\\
V_m(\boldsymbol{\delta})
&=\Pr_{a\sim P_{\mathrm{eval}}}
\!\left(\Delta_m(a)>\delta_{ij(a)}\right).
\end{aligned}
\end{equation}
Let $w_{ij}=\Pr_{a\sim P_{\mathrm{eval}}}(ij(a)=(i,j))$.

\begin{reliftheorem}[Threshold monotonicity and Lipschitzness]
\label{thm:threshold-lipschitz}
For any method $m$, if $\boldsymbol{\delta}\le \boldsymbol{\delta}'$ component-wise, then
$B_m(\boldsymbol{\delta}')\le B_m(\boldsymbol{\delta})$ and
$V_m(\boldsymbol{\delta}')\le V_m(\boldsymbol{\delta})$. Moreover,
\begin{equation}
|B_m(\boldsymbol{\delta})-B_m(\boldsymbol{\delta}')|
\le
\sum_{i<j}w_{ij}|\delta_{ij}-\delta'_{ij}|.
\end{equation}
\end{reliftheorem}

\begin{proof}
For a fixed tuple $a$, $(\Delta_m(a)-\delta)_+$ and
$\mathbf{1}\{\Delta_m(a)>\delta\}$ are non-increasing in $\delta$, so taking expectation gives the monotonicity statements. For the Lipschitz bound, the map $\delta\mapsto (x-\delta)_+$ is 1-Lipschitz for every $x\in[0,1]$. Therefore
\begin{equation}
\begin{aligned}
|B_m(\boldsymbol{\delta})-B_m(\boldsymbol{\delta}')|
&\le
\mathbb E_a\!\left[|\delta_{ij(a)}-\delta'_{ij(a)}|\right]  \\
&=
\sum_{i<j}w_{ij}|\delta_{ij}-\delta'_{ij}|.
\end{aligned}
\end{equation}
\end{proof}

\begin{reliftheorem}[Ranking robustness under threshold drift]
\label{thm:ranking-robustness}
Consider two methods $A$ and $B$. Let $\boldsymbol{\delta}_{\mathrm{ref}}$ be the fixed reference threshold vector and let $\boldsymbol{\delta}_A,\boldsymbol{\delta}_B$ be method-induced threshold vectors. Define
\begin{equation}
\varepsilon_m
=
\sum_{i<j}w_{ij}|\delta_{m,ij}-\delta_{\mathrm{ref},ij}|.
\end{equation}
If
\begin{equation}
|B_A(\boldsymbol{\delta}_{\mathrm{ref}})
-B_B(\boldsymbol{\delta}_{\mathrm{ref}})|
>
\varepsilon_A+\varepsilon_B,
\end{equation}
then replacing the fixed threshold by the method-induced thresholds cannot reverse the Bias ordering of $A$ and $B$. Conversely, if the ordering changes, then the fixed-audit gap is at most $\varepsilon_A+\varepsilon_B$.
\end{reliftheorem}

\begin{proof}
By Theorem~\ref{thm:threshold-lipschitz},
\begin{equation}
|B_A(\boldsymbol{\delta}_A)-B_A(\boldsymbol{\delta}_{\mathrm{ref}})|\le \varepsilon_A,\quad
|B_B(\boldsymbol{\delta}_B)-B_B(\boldsymbol{\delta}_{\mathrm{ref}})|\le \varepsilon_B.
\end{equation}
Thus the raw and fixed Bias gaps differ by at most $\varepsilon_A+\varepsilon_B$. If $|B_A(\boldsymbol{\delta}_{\mathrm{ref}})-B_B(\boldsymbol{\delta}_{\mathrm{ref}})|$ exceeds $\varepsilon_A+\varepsilon_B$, this perturbation cannot change its sign. The converse follows by contraposition.
\end{proof}

\begin{reliftheorem}[Reference-threshold calibration]
\label{thm:conformal-delta}
Fix a reference model or representation and let $D_1,\ldots,D_n$ be exchangeable reference-pair distances in $[0,1]$. For $\alpha\in(0,1)$, set $k=\lceil(n+1)(1-\alpha)\rceil$ and choose $\widehat{\delta}_{\mathrm{ref}}=D_{(k)}$ if $k\le n$, with the conservative convention $\widehat{\delta}_{\mathrm{ref}}=1$ otherwise. For a new exchangeable reference pair with distance $D_{n+1}$,
\begin{equation}
\Pr(D_{n+1}>\widehat{\delta}_{\mathrm{ref}})
\le
\frac{n+1-k}{n+1}
\le \alpha .
\end{equation}
\end{reliftheorem}

\begin{proof}
This is the standard split-conformal rank argument~\citep{vovk2005algorithmic,angelopoulos2022conformal}. Without ties, the rank of $D_{n+1}$ among the $n+1$ exchangeable scores is uniform. The event $D_{n+1}>D_{(k)}$ occurs only when the rank is larger than $k$, which has probability $(n+1-k)/(n+1)$. With ties, the strict inequality is conservative.
\end{proof}

Theorem~\ref{thm:conformal-delta} does not state that every method's fairness violation rate is bounded by $\alpha$. It states that $\delta_{\mathrm{ref}}$ can be selected as a calibrated semantic tolerance for a reference pair distribution and then held fixed for all methods, which is the role needed by the fixed-$\delta$ audit.

\begin{reliftheorem}[Huberized hinge approximation]
\label{thm:huber-hinge-approx}
Let $H(z)=(z)_+$ and let $\phi_\mu(z)$ be the huberized hinge in Eq.~(\ref{eq:huber-hinge}). For all $z\in\mathbb R$,
\begin{equation}
0\le H(z)-\phi_\mu(z)\le \frac{\mu}{2}.
\end{equation}
Consequently, for any random margin $z$ and any $\gamma>0$,
\begin{equation}
\Pr(z>\gamma)
\le
\frac{\mathbb E[\phi_\mu(z)]+\mu/2}{\gamma}.
\end{equation}
\end{reliftheorem}

\begin{proof}
If $z\le0$, both $H(z)$ and $\phi_\mu(z)$ are zero. If $0<z\le\mu$, then
$H(z)-\phi_\mu(z)=z-z^2/(2\mu)$, which lies in $[0,\mu/2]$. If $z>\mu$, then
$H(z)-\phi_\mu(z)=\mu/2$. Taking expectations gives
$\mathbb E[H(z)]\le\mathbb E[\phi_\mu(z)]+\mu/2$, and Markov's inequality applied to $H(z)$ gives the tail bound.
\end{proof}

For the raw unsmoothed ratios, dividing the confidence-gap magnitude and the clipped distance by their own 95th-percentile scales cancels purely multiplicative scale changes when clipping is inactive. The implemented margin applies the fixed-$\beta$ Huber transform before division, so it is scale-normalized but not exactly scale-invariant, especially in the quadratic Huber regime.

The following log-space recursion is the first-order idealization of the implemented multiplicative update when the multiplicative factor remains positive and close to one. The statement characterizes the intended local regime rather than establishing global convergence of the deep network.

\begin{reliftheorem}[Local controller stability]
\label{thm:controller-stability}
Consider the log-penalty idealization of the multiplicative controller, with $u_t=\log\lambda_t$ and projected update
\begin{equation}
u_{t+1}
=
\Pi_{[u_{\min},u_{\max}]}\!\left(u_t+\eta(h(u_t)+\xi_t)\right),
\end{equation}
where $h(u)=\mathbb E[r_t-r^\star\mid u_t=u]$, $\mathbb E[\xi_t\mid\mathcal F_t]=0$, and
$\mathbb E[\xi_t^2\mid\mathcal F_t]\le\sigma^2$. Suppose there is $u^\star\in[u_{\min},u_{\max}]$ with $h(u^\star)=0$, and on the interval of interest
\begin{equation}
(u-u^\star)h(u)\le-\kappa(u-u^\star)^2,\qquad
|h(u)|\le L|u-u^\star|.
\end{equation}
If $0<\eta<2\kappa/L^2$, then with
$\varrho=1-2\eta\kappa+\eta^2L^2<1$,
\begin{equation}
\mathbb E[(u_t-u^\star)^2]
\le
\varrho^t(u_0-u^\star)^2
+
\frac{\eta^2\sigma^2}{1-\varrho},
\end{equation}
and
\begin{equation}
\limsup_{t\to\infty}\mathbb E[(u_t-u^\star)^2]
\le
\frac{\eta\sigma^2}{2\kappa-\eta L^2}.
\end{equation}
\end{reliftheorem}

\begin{proof}
Let $\Delta_t=u_t-u^\star$. Projection onto an interval is non-expansive, so conditioning on $\mathcal F_t$ gives
\begin{equation}
\begin{aligned}
\mathbb E[\Delta_{t+1}^2\mid\mathcal F_t]
&\le
\mathbb E[(\Delta_t+\eta h(u_t)+\eta\xi_t)^2\mid\mathcal F_t] \\
&=
\Delta_t^2+2\eta\Delta_t h(u_t)+\eta^2 h(u_t)^2+\eta^2\mathbb E[\xi_t^2\mid\mathcal F_t] \\
&\le
(1-2\eta\kappa+\eta^2L^2)\Delta_t^2+\eta^2\sigma^2.
\end{aligned}
\end{equation}
Iterating the recursion yields the finite-time and steady-state bounds.
\end{proof}

Theorem~\ref{thm:controller-stability} is local: it assumes a feasible target, monotone expected response, bounded noise, and clipping to the operational interval. It formalizes the intended controller regime and explains why overly large $\eta$ can move training to a less favorable operating point, matching the local sensitivity checks below.

\subsection{MIMIC-III fixed-threshold sensitivity}
\begin{table}[H]
\centering
\scriptsize
\setlength{\tabcolsep}{2.2pt}
\renewcommand{\arraystretch}{1.02}
\caption{MIMIC-III fixed-$\delta$ sensitivity. Entries are Bias$_{\mathrm{fixed}\text{-}\delta}$ mean$\pm$std over 7 runs on the test split.}
\label{tab:delta-sensitivity}
\begin{tabular}{lccccc}
\toprule
Method & $0.15$ & $0.20$ & $0.275$ & $0.35$ & $0.45$ \\
\midrule
ERM & $.321\pm.019$ & $.283\pm.018$ & $.227\pm.017$ & $.176\pm.016$ & $.119\pm.013$ \\
Uncert. & $\underline{.302\pm.035}$ & $\underline{.263\pm.034}$ & $\underline{.209\pm.032}$ & $\underline{.161\pm.029}$ & $\underline{.108\pm.022}$ \\
GradNorm & $\mathbf{.283\pm.012}$ & $\mathbf{.246\pm.012}$ & $\mathbf{.192\pm.011}$ & $\mathbf{.144\pm.010}$ & $\mathbf{.093\pm.008}$ \\
PCGrad & $.339\pm.015$ & $.300\pm.014$ & $.244\pm.014$ & $.192\pm.013$ & $.132\pm.011$ \\
FairGrad & $.320\pm.014$ & $.282\pm.013$ & $.229\pm.013$ & $.180\pm.011$ & $.126\pm.010$ \\
CATS & $.333\pm.056$ & $.293\pm.055$ & $.237\pm.052$ & $.185\pm.047$ & $.128\pm.039$ \\
ReLiF & $.313\pm.016$ & $.275\pm.016$ & $.220\pm.015$ & $.169\pm.014$ & $.113\pm.011$ \\
\bottomrule
\end{tabular}
\end{table}

\subsection{Controller sensitivity}
\label{app:controller-sensitivity}
The controller checks are local sensitivity analyses on MIMIC-III rather than new main results. With $r^\star=0.16$, the tested step sizes $\eta=10^{-4},10^{-3},10^{-2}$ give fixed-$\delta$ Bias values of $0.2158$, $0.1889$, and $0.2626$, respectively. All runs complete training, but the aggressive $\eta=10^{-2}$ setting moves to a less favorable fixed-threshold operating point. A single-seed sweep of $r^\star$ from $0.16$ to $0.28$ in increments of $0.02$ keeps Worst, Macro, and Bias in the same range, with larger $r^\star$ values more often touching the lower $\lambda$ bound. This supports the weaker but more appropriate conclusion that the controller is tunable rather than brittle.

\subsection{NYUv2 violation rates}
Table~\ref{tab:nyuv2-vr} reports violation rates under the same fixed-threshold audit as Table~\ref{tab:nyuv2-main}. The pattern mirrors the aligned-bias results: non-ReLiF baselines remain around $0.68$--$0.71$ at $\delta_{p75}$, whereas ReLiF drops to $0.0862$, with the same separation under $\delta_{p50}$.
\begin{table}[!htbp]
\centering
\scriptsize
\setlength{\tabcolsep}{4pt}
\renewcommand{\arraystretch}{1.03}
\caption{NYUv2 fixed-threshold violation rates under the same audit as Table~\ref{tab:nyuv2-main}. Values are mean$\pm$sample standard deviation over 7 seeds.}
\label{tab:nyuv2-vr}
\begin{tabular}{lcc}
\toprule
Method & VR@$\delta_{p75}$ $\downarrow$ & VR@$\delta_{p50}$ $\downarrow$ \\
\midrule
ERM & $.6896\pm.0257$ & $.7084\pm.0282$ \\
Uncertainty & $.6925\pm.0305$ & $.7108\pm.0286$ \\
GradNorm & $.6954\pm.0372$ & $.7153\pm.0312$ \\
PCGrad & $.7087\pm.0424$ & $.7285\pm.0372$ \\
FairGrad & $\underline{.6835\pm.0532}$ & $\underline{.7037\pm.0459}$ \\
CATS & $.7095\pm.0437$ & $.7286\pm.0416$ \\
ReLiF & $\mathbf{.0862\pm.0964}$ & $\mathbf{.1127\pm.1201}$ \\
\bottomrule
\end{tabular}
\end{table}

\FloatBarrier
\end{document}